\documentclass{article}

\usepackage[nonatbib, final]{neurips_2018}

\usepackage[utf8]{inputenc} 
\usepackage[T1]{fontenc}    
\usepackage{hyperref}       
\usepackage{url}            
\usepackage{booktabs}       
\usepackage{amsfonts}       
\usepackage{nicefrac}       
\usepackage{microtype}      
\usepackage{amsthm}
\usepackage{amsmath}
\usepackage{amssymb}
\usepackage{graphicx}
\usepackage{booktabs}
\usepackage{makecell}
\usepackage{enumitem, subcaption}
\usepackage{color}
\usepackage{wrapfig}

\title{Adversarial vulnerability for any classifier}

\newtheorem{theorem}{Theorem}
\newtheorem{lemma}{Lemma}

\newtheorem{example}{Example}

\newtheorem{assumption}{Assumption}

\newcommand{\Z}{\mathcal Z}

\newcommand{\cX}{\mathcal{X}}

\newcommand{\ex}[2]{\underset{#1}{\mathbb{E}}#2}

\newcommand{\RR}{\mathbb{R}}

\newcommand{\ZZ}{\mathbb{Z}}

\newcommand{\E}{\mathbb{E}}

\newcommand{\floor}[1]	{\left\lfloor #1 \right\rfloor}

\newcommand{\eps}{\eta}

\DeclareMathOperator{\dist}{dist}

\newcommand{\nc}{\newcommand}
\nc{\rnc}{\renewcommand}
\nc{\lbar}[1]{\overline{#1}}
\nc{\bra}[1]{\langle#1|}
\nc{\ket}[1]{|#1\rangle}
\nc{\ketbra}[2]{|#1\rangle\!\langle#2|}
\nc{\braket}[2]{\langle#1|#2\rangle}
\nc{\proj}[1]{| #1\rangle\!\langle #1 |}
\nc{\avg}[1]{\langle#1\rangle}
\nc{\smfrac}[2]{\mbox{$\frac{#1}{#2}$}}
\nc{\tr}{\operatorname{tr}}
\nc{\pd}{\mathrm{P}}
\nc{\gtd}{\mathrm{D}}
\nc{\gfid}{\bar{\mathrm{F}}}
\graphicspath{{figs/}}

\usepackage{algorithm,algpseudocode}
\newcommand{\PP}{\mathbb{P}}

\author{
  Alhussein Fawzi \\
  DeepMind\\
  \texttt{afawzi@google.com} \\
  \And
  Hamza Fawzi \\
  Department of Applied Mathematics\\
  \& Theoretical Physics\\
  University of Cambridge \\
  \texttt{h.fawzi@damtp.cam.ac.uk} \\
  \And
  Omar Fawzi \\
  ENS de Lyon\thanks{Univ Lyon, ENS de Lyon, CNRS, UCBL,
LIP, F-69342, Lyon Cedex 07, France} \\
  \texttt{omar.fawzi@ens-lyon.fr} \\
}

\begin{document}

\maketitle

\begin{abstract}
	Despite achieving impressive performance, state-of-the-art classifiers remain highly vulnerable to small, imperceptible, adversarial perturbations.  This vulnerability has proven empirically to be very intricate to address. In this paper, we study the phenomenon of adversarial perturbations under the assumption that the data is generated with a smooth generative model. We derive fundamental upper bounds on the robustness to perturbations of any classification function, and prove the existence of adversarial perturbations that transfer well across different classifiers with small risk. Our analysis of the robustness also provides insights onto key properties of generative models, such as their smoothness and dimensionality of latent space. We conclude with numerical experimental results showing that our bounds provide informative baselines to the maximal achievable robustness on several datasets.
\end{abstract}

	\section{Introduction}
	Deep neural networks are powerful models that
	achieve state-of-the-art performance  across several domains, such as bioinformatics \cite{bio2, bio1}, speech \cite{sp2}, and computer vision \cite{he2015deep, cv2}. Though deep networks have exhibited very good performance in classification tasks, they have recently been shown to be  unstable to adversarial perturbations of the data \cite{szegedy2013intriguing, biggio2013evasion}. In fact, very small and often imperceptible perturbations of the data samples are sufficient to fool state-of-the-art classifiers and result in incorrect classification. This discovery of the surprising vulnerability of classifiers to perturbations has led to a large body of work that attempts to design robust classifiers \cite{goodfellow2014, shaham2015understanding, madry2017towards, cisse2017parseval, papernot2015distillation, alemi2016deep}. However, advances in designing robust classifiers have been accompanied with stronger perturbation schemes that defeat such defenses  \cite{carlini2017adversarial, uesato2018adversarial, robust_vision}. 
	
	In this paper, we assume that the data distribution is defined by a smooth generative model (mapping latent representations to images), and study theoretically the existence of small adversarial perturbations for arbitrary classifiers. We summarize our main contributions as follows:
	\begin{itemize}
	    \item We show fundamental upper bounds on the robustness of any classifier to perturbations, which provides a baseline to the maximal achievable robustness. When the latent space of the data distribution is high dimensional, our analysis shows that \textit{any} classifier is vulnerable to very small perturbations. Our results further suggest the existence of a tight relation between robustness and linearity of the classifier in the latent space.
	    \item We prove the existence of adversarial perturbations that transfer across different classifiers. This provides theoretical justification to previous empirical findings that highlighted the existence of such transferable perturbations.
	    \item We quantify the difference between the robustness to adversarial examples \textit{in the data manifold} and \textit{unconstrained} adversarial examples, and show that the two notions of robustness can be precisely related: for any classifier $f$ with in-distribution robustness $r$, there exists a classifier $\tilde{f}$ that achieves unconstrained robustness $r/2$. This further provides support to the empirical observations in \cite{ilyas2017robust, gilmer2018adversarial}.
	    \item We evaluate our bounds in several experimental setups (CIFAR-10 and SVHN), and show that they yield informative baselines to the maximal achievable robustness. %
	\end{itemize}
	Our robustness analysis provides in turn insights onto desirable properties of generative models capturing real-world distributions. 
	In particular, the intriguing generality of our analysis implies that when the data distribution is modeled through a smooth and generative model with high-dimensional latent space, there exist small-norm perturbations of images that fool humans for any discriminative task defined on the data distribution. If, on the other hand, it is the case that the human visual system is inherently robust to small perturbations (e.g., in $\ell_p$ norm), then our analysis shows  that a distribution over natural images cannot be modeled by smooth and high-dimensional generative models. Going forward in modeling complex natural image distributions, our results hence suggest that low dimensional, non-smooth generative models are important constraints to capture the real-world distribution of images; not satisfying such constraints can lead to small adversarial perturbations for any classifier, including the human visual system.

	\section{Related work}
	
    It was proven in \cite{fawzi2015a, nips2016_ours} that for certain families of classifiers, there exist adversarial perturbations that cause misclassification of magnitude $O(1/\sqrt{d})$, where $d$ is the data dimension, provided the robustness to random noise is fixed (which is typically the case if e.g., the data is normalized). 
	In addition, fundamental limits on the robustness of classifiers were derived in \cite{fawzi2015a} for some simple classification families. Other works have instead studied the existence of adversarial perturbations, under strong assumptions on the data distribution \cite{gilmer2018adversarial, tanay2016boundary}. In this work, motivated by the success of generative models mapping latent representations with a normal prior, we instead study the existence of robust classifiers under this general data-generating procedure and derive bounds on the robustness that hold for any classification function. A large number of techniques have recently been proposed to improve the robustness of classifiers to perturbations, such as adversarial training \cite{goodfellow2014}, robust optimization \cite{shaham2015understanding, madry2017towards}, regularization  \cite{cisse2017parseval}, distillation \cite{papernot2015distillation},  stochastic networks \cite{alemi2016deep}, etc... Unfortunately, such techniques have been shown to fail whenever a more complex attack strategy is used \cite{carlini2017adversarial, uesato2018adversarial}, or when it is evaluated on a more complex dataset. Other works have recently studied procedures and algorithms to provably guarantee a certain level of robustness \cite{hein2017formal, peck2017lower, sinha2017certifiable, raghunathan2018certified, dvijotham2018dual}, and have been applied to small datasets (e.g., MNIST). 
    For large scale, high dimensional datasets, the problem of designing robust classifiers is entirely open. We finally note that adversarial examples for generative models have recently been considered in \cite{kos2017adversarial}; our aim here is however different as our goal is to bound the robustness of classifiers when data comes from a generative model.
	
	\section{Definitions and notations}
	\label{sec:def_notations}

    Let $g$ be a generative model that maps latent vectors $z \in \mathcal{Z} := \RR^d$ to the space of images $\mathcal{X} := \mathbb{R}^m$, with $m$ denoting the number of pixels. To generate an image according to the distribution of natural images $\mu$, we generate a random vector $z \sim \nu$ according to the standard Gaussian distribution $\nu = \mathcal{N} (0, I_d)$, and we apply the map $g$; the resulting image is then $g(z)$.
	 This data-generating procedure is motivated by numerous previous works on generative models, whereby natural-looking images are obtained by transforming normal vectors through a deep neural network \cite{kingma2013auto}, \cite{goodfellow2014generative}, \cite{radford2015unsupervised}, \cite{arjovsky2017wasserstein}, \cite{gulrajani2017improved}.\footnote{Instead of sampling from $\mathcal{N}(0,I_d)$ in $\mathcal{Z}$, some generative models sample from the uniform distribution in $[-1, 1]^d$. The results of this paper can be easily extended to such generative procedures.} Let $f: \mathbb{R}^m \rightarrow \{1,\dots,K\}$ be a classifier mapping images in $\RR^m$ to discrete labels $\{1,\dots,K\}$. The discriminator $f$ partitions $\mathcal{X}$ into $K$ sets $C_i = \{x \in \cX : f(x) = i\}$ each of which corresponds to a different predicted label.
     The relative proportion of points in class $i$ is equal to $\mathbb{P}(C_i) = \nu(g^{-1}(C_i))$, the Gaussian measure of $g^{-1}(C_i)$ in $\mathcal{Z}$. 
	 
	 The goal of this paper is to study the \emph{robustness} of $f$ to additive perturbations under the assumption that the data is generated according to $g$. 
	 We define two notions of robustness. These effectively measure the minimum distance one has to travel in image space to change the classification decision. 
		\begin{itemize}
		\item \textbf{In-distribution robustness:}
		For $x = g(z)$, we define the in-distribution robustness $r_{\text{in}} (x)$ as follows:
		\[
		r_{\text{in}} (x) = \min_{r \in \mathcal{Z}} \| g(z + r) - x \| \text{ s.t. } f(g(z+r)) \neq f(x),
		\]
		where $\| \cdot \|$ denotes an arbitrary norm on $\mathcal{X}$. Note that the perturbed image, $g(z+r)$ is \textit{constrained to lie in the image} of $g$, and hence belongs to the support of the distribution $\mu$. 
		\item \textbf{Unconstrained robustness:} Unlike the in-distribution setting, we measure here the robustness to \textit{arbitrary} perturbations in the image space; that is, the perturbed image is not constrained anymore to belong to the data distribution $\mu$. 
		$$r_{\text{unc}} (x) = \min_{r \in \mathcal{X}} \| r \| \text{ s.t. } f(x+r) \neq f(x).$$
		This notion of robustness corresponds to the widely used definition of adversarial perturbations. It is easy to see that this robustness definition is smaller than the in-distribution robustness; i.e., $r_{\text{unc}} (x) \leq r_{\text{in}} (x)$. 
	\end{itemize}

	In this paper, we assume that the generative model is smooth, in the sense that it satisfies a \textit{modulus of continuity} property, defined as follows:
	\begin{assumption}
	 We assume that $g$ admits a monotone invertible modulus of continuity $\omega$; i.e.,\footnote{This assumption can be extended to random $z$ (see C.2 in the appendix). For ease of exposition however, we use here the deterministic assumption.}
	\begin{equation}
	\label{eq:modulus_cont}
	\forall z, z' \in \Z, \| g(z) - g(z') \| \leq \omega(\| z - z' \|_2).
	\end{equation}	
	\end{assumption}
	Note that the above assumption is milder than assuming Lipschitz continuity. In fact, the Lipschitz property corresponds to choosing $\omega(t)$ to be a linear function of $t$. In particular, the above assumption does not require that  $\omega(0) = 0$, which potentially allows us to model distributions with disconnected support.\footnote{In this paper, we use the term \textit{smooth} generative models to denote that the function $\omega(\delta)$ takes small values for small $\delta$.}  
	
	It should be noted that generator smoothness is a desirable property of generative models. This property is often illustrated empirically by generating images along a straight path in the latent space \cite{radford2015unsupervised}, and verifying that the images undergo gradual semantic changes between the two endpoints. In fact, smooth transitions is often used as a qualitative evidence that the generator has learned relevant factors of variation. 

	Fig. \ref{fig:illustration_generative} summarizes the problem setting and notations. Assuming that the data is generated according to $g$, we analyze in the remainder of the paper the robustness of arbitrary classifiers to perturbations.

	\begin{figure}
	    \centering
	    \includegraphics[scale=0.6]{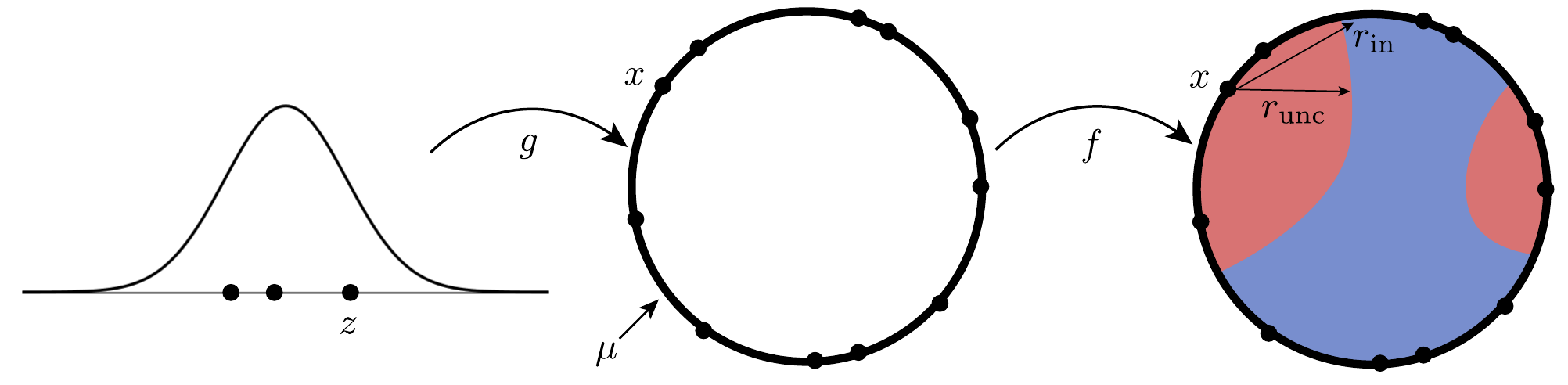}
	    \caption{\label{fig:illustration_generative}Setting used in this paper. The data distribution is obtained by mapping $\mathcal{N} (0, I_d)$ through $g$ (we set $d = 1$ and $g(z) = (\cos(2 \pi z), \sin(2 \pi z))$ in this example). The thick circle indicates the support of the data distribution $\mu$ in $\mathbb{R}^m$ ($m = 2$ here). The binary discriminative function $f$ separates the data space into two classification regions (red and blue colors). While the in-distribution perturbed image is required to belong to the data support, this is not necessarily the case in the unconstrained setting. In this paper, we do not put any assumption on $f$, resulting in potentially arbitrary partitioning of the data space. While the existence of very small adversarial perturbations seems counter-intuitive in this low-dimensional illustrative example (i.e., $r_{\text{in}}$ and $r_{\text{unc}}$ can be large for some choices of $f$), we show in the next sections that this is the case in high dimensions.}
	\end{figure}

	\section{Analysis of the robustness to perturbations}

    \subsection{Upper bounds on robustness}

    We state a general bound on the robustness to perturbations and derive two special cases to make more explicit the dependence on the distribution 
    and number of classes.    
    
	\begin{theorem}
		\label{thm:image_space_bounds}
		Let $f: \mathbb{R}^m \rightarrow \{1, \dots, K\}$ be an arbitrary classification function defined on the image space. Then, the fraction of datapoints having robustness less than $\eta$ satisfies:
		\begin{align}
		\label{eq:main_theorem_image_space}
		\mathbb{P} \left( r_{\text{in}} (x) \leq \eta \right) \geq \sum_{i=1}^{K} (\Phi(a_{\neq i} + \omega^{-1}(\eta)) - \Phi(a_{\neq i})) \ ,
		\end{align}
		where $\Phi$ is the cdf of $\mathcal{N}(0,1)$, and $a_{\neq i} = \Phi^{-1} \left(\mathbb{P} \left(\bigcup\limits_{j \neq i} C_j \right)\right)$. 
        
        In particular, if for all $i$, $\mathbb{P}(C_i) \leq \frac{1}{2}$ (the classes are not too unbalanced), we have
		\begin{align}
		\label{eq:prob_bound_half}
		\mathbb{P} \left(  r_{\text{in}} (x)  \leq \eta \right) \geq 1 - \sqrt{\frac{\pi}{2}} e^{-\omega^{-1}(\eps)^2/2} \ .
		\end{align}
		
		To see the dependence on the number of classes more explicitly, consider the setting where the classes are equiprobable, i.e., $\mathbb{P} (C_i) = \frac{1}{K}$ for all $i$, $K \geq 5$, then 
		\begin{align}
		\label{eq:prob_bound_class_dependence}
		\mathbb{P} \left( r_{\text{in}} (x) \leq \eta \right) &\geq 1 - \sqrt{\frac{\pi}{2}} e^{-\omega^{-1}(\eps)^2/2} 
		e^{-\eps \sqrt{\log\left(\frac{K^2}{4\pi \log(K)}\right)}} \ .
		\end{align}
	\end{theorem}

	This theorem is a consequence of the Gaussian isoperimetric inequality first proved in~\cite{borell1975brunn} and~\cite{sudakov1978extremal}. The  proofs can be found in the appendix. 

\vspace{2mm}

	\textbf{Remark 1. Interpretation.} For easiness of interpretation, we assume that the function $g$ is Lipschitz continuous, in which case $\omega^{-1}(\eps)$ is replaced with $\eps/L$ where $L$ is the Lipschitz constant. Then, Eq. (\ref{eq:prob_bound_half}) shows the existence of perturbations of norm $\eps \propto L$ that can fool any classifier. This norm should be compared to the typical norm given by  $\mathbb{E} \| g(z) \|$. By normalizing the data, we can assume $\mathbb{E} \| g(z) \| = \mathbb{E} \| z \|_2$ without loss of generality.\footnote{Without this assumption, the following discussion applies if we replace the Lipschitz constant with the normalized Lipschitz constant $L' = L \frac{\mathbb{E} \| z \|_2}{\mathbb{E} \| g(z) \|}$.} As $z$ has a normal distribution, we have $\mathbb{E} \| z \|_2 \in [\sqrt{d-1}, \sqrt{d}]$ and thus the typical norm of an element in the data set satisfies $\mathbb{E} \| g(z) \| \geq \sqrt{d-1}$. Now if we plug in $\eta = 2 L$, we obtain that the robustness is less than $2 L$ with probability exceeding 0.8. This should be compared to the typical norm which is at least $\sqrt{d-1}$. Our result therefore shows that when $d$ is large and $g$ is smooth (in the sense that $L \ll \sqrt{d}$), there exist small adversarial perturbations that can fool arbitrary classifiers $f$. Fig. \ref{fig:high_probability_d} provides an  illustration of the upper bound, in the case where $\omega$ is the identity function.
	
	\textbf{Remark 2. Dependence on $K$.} Theorem \ref{thm:image_space_bounds} shows an \textit{increasing} probability of misclassification with the number of classes $K$. In other words, it is easier to find adversarial perturbations in the setting where the number of classes is large, than for a binary classification task.\footnote{We assume here equiprobable classes.} This dependence confirms empirical results whereby the robustness is observed to decrease with the number of classes. The dependence on $K$ captured in our bounds is in contrast to previous bounds that showed decreasing probability of fooling the classifier, for larger number of classes \cite{nips2016_ours}.

	\textbf{Remark 3. Classification-agnostic bound.} Our bounds hold for any classification function $f$, and are not specific to a family of classifiers. This is unlike the work of~\cite{fawzi2015a} that establishes bounds on the robustness for specific classes of functions (e.g., linear or quadratic classifiers).
	
	\textbf{Remark 4. How tight is the upper bound on robustness in Theorem \ref{thm:image_space_bounds}?} Assuming that the smoothness assumption in Eq. \ref{eq:modulus_cont} is an equality, let the classifier $f$ be such that $f \circ g$ separates the latent space into $B_1 = g^{-1}(C_1) = \{ z: z_1 \geq 0 \}$ and $B_2 = g^{-1}(C_2) = \{ z: z_1 < 0 \}$. Then, it follows that 
	\begin{align*}
	\mathbb{P} (r_{\text{in}} (x)) \leq \eta) & = \mathbb{P} (\exists r: \| g(z+r) - g(z) \| \leq \eta, f(g(z+r)) \neq f(g(z)))) \\
	                                          & = \mathbb{P} (\exists r: \| r \|_2  \leq \omega^{-1}(\eta), \text{sgn} (z_1 + r_1) \text{sgn} (z_1) < 0) \\ 
	                                          & = \mathbb{P} (z \in B_1,  z_1 < \omega^{-1}(\eta)) + \mathbb{P} (z \in B_2,  z_1 \geq -\omega^{-1}(\eta)) = 2 (\Phi(\omega^{-1}(\eta)) - \Phi(0)),
	\end{align*}
	which precisely corresponds to Eq. (\ref{eq:main_theorem_image_space}). In this case, the bound in Eq. (\ref{eq:main_theorem_image_space}) is therefore an equality. More generally, this bound is an equality if the classifier induces linearly separable regions in the latent space.\footnote{In the case where Eq. (\ref{eq:modulus_cont}) is an inequality, we will not exactly achieve the bound, but get closer to it when $f \circ g$ is linear.} This suggests that classifiers are maximally robust when the induced classification boundaries in the latent space are linear. We stress on the fact that boundaries in the $\mathcal{Z}$-space can be very different from the boundaries in the image space. In particular, as $g$ is in general non-linear, $f$ might be a highly \textit{non-linear function} of the input space, while $z \mapsto (f \circ g) (z)$ is a linear function in $z$. We provide an explicit example in the appendix illustrating this remark.
	
   \begin{figure}[ht]
	\centering
	\includegraphics[width=0.4\textwidth]{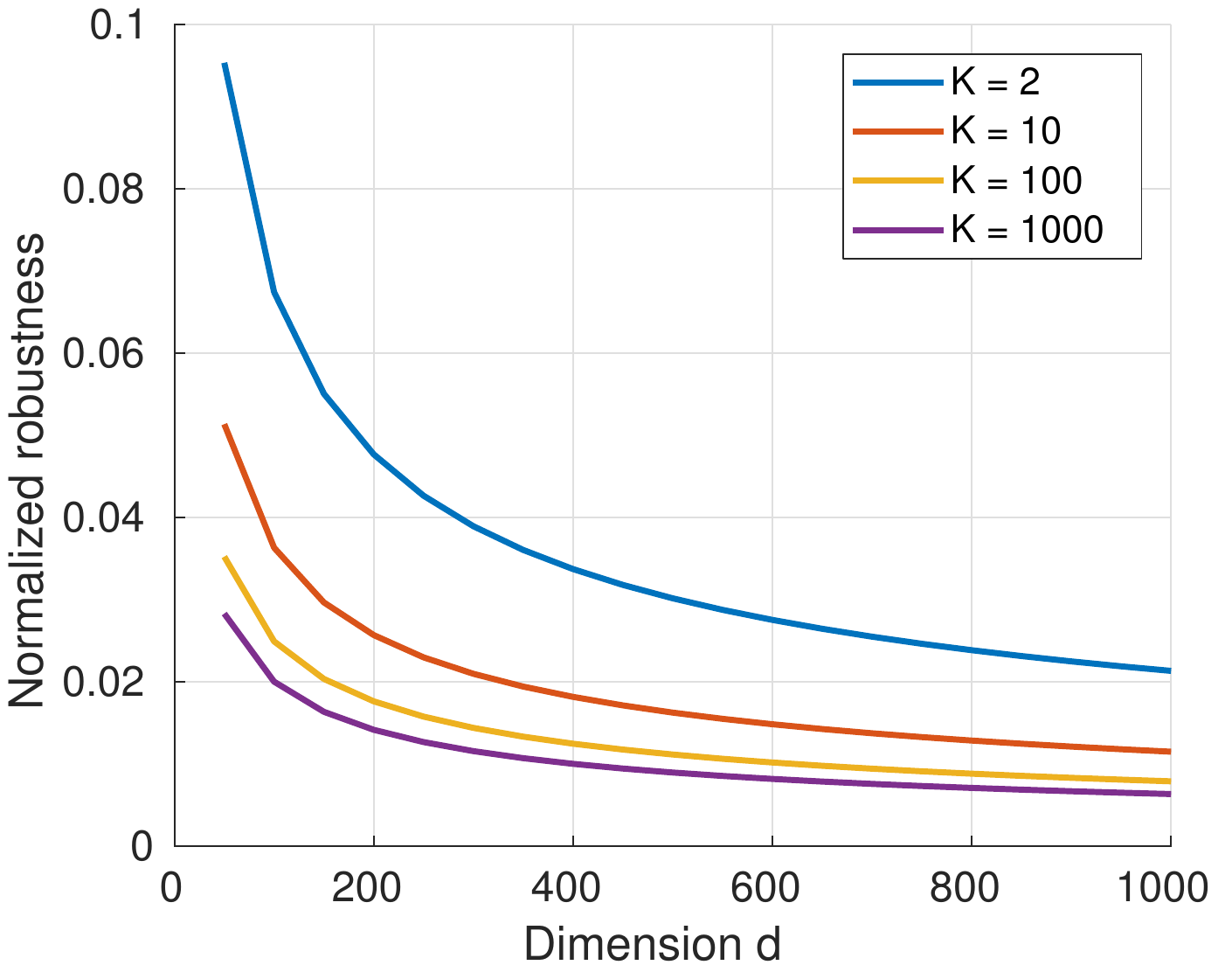}
	\caption{\label{fig:high_probability_d} Upper bound (Theorem \ref{thm:image_space_bounds}) on the median of the normalized robustness $r_{\text{in}} / \sqrt{d}$ for different values of the number of classes $K$, in the setting where  $\omega(t) = t$. We assume that classes have equal measure (i.e., $\mathbb{P}(C_i) = 1/K$).}
    \end{figure}

	\textbf{Remark 5. Adversarial perturbations in the latent space} While the quantities introduced in Section \ref{sec:def_notations} measure the robustness in the \textit{image space}, an alternative is to measure the robustness in the \textit{latent space}, defined as $r_{Z} = \min_{r} \| r \|_2 \text{ s.t. } f(g(z+r)) \neq f(g(z))$. For natural images, latent vectors provide a decomposition of images into meaningful factors of variation, such as features of objects in the image, illumination, etc... Hence, perturbations of vectors in the latent space measure the amount of change one needs to apply to such meaningful latent features to cause data misclassification. A bound on the magnitude of the minimal perturbation in the latent space (i.e., $r_{Z}$) can be directly obtained from Theorem \ref{thm:image_space_bounds} by setting $\omega$ to identity (i.e., $\omega(t) = t$). Importantly, note that no assumptions on the smoothness of the generator $g$ are required for our bounds to hold when considering this notion of robustness.
    
	\textbf{Relation between in-distribution robustness and unconstrained robustness.}
	
	While the previous bound is specifically looking at the in-distribution robustness, in many cases, we are interested in achieving \textit{unconstrained} robustness; that is, the perturbed image is not constrained to belong to the data distribution (or equivalently to the range of $g$). It is easy to see that any bound derived for the in-distribution robustness $ r_{\text{in}} (x)$ also holds for the unconstrained robustness $r_{\text{unc}}(x)$ since it clearly holds that $r_{\text{unc}}(x) \leq r_{\text{in}} (x)$. One may wonder whether it is possible to get a better upper bound on $r_{\text{unc}}(x)$ directly. We show here that this is not possible if we require our bound to hold for any general classifier. Specifically, we construct a family of classifiers for which $r_{\text{unc}}(x) \geq \frac{1}{2} r_{\text{in}} (x)$, which we now present:

	For a given classifier $f$ in the image space, define the classifier $\tilde{f}$ constructed in a nearest neighbour strategy:
	\begin{align}
	\label{eq:classiifer_h}
	\tilde{f}(x) = f(g(z^*)) \quad \text{ with } \quad z^* = \arg\min_{z} \| g(z) - x \|.
	\end{align}
	Note that $\tilde{f}$ behaves exactly in the same way as $f$ on the image of $g$ (in particular, it has the same risk and in-distribution robustness). We show here that it has an unconstrained robustness that is at least half of the in-distribution robustness of $f$.
	
	\begin{theorem}
		\label{thm:relation_two_robustness}
		For the classifier $\tilde{f}$, we have $r_{\text{unc}}(x) \geq \frac{1}{2} r_{\text{in}} (x)$.
	\end{theorem}
	
	This result shows that if a classifier has in-distribution robustness $r$, then we can construct a classifier with unconstrained robustness $r/2$, through a simple modification of the original classifier $f$. 
	Hence, classification-agnostic limits derived for both notions of robustness are essentially the same. It should further be noted that the procedure in Eq. (\ref{eq:classiifer_h}) provides a constructive method to increase the robustness of any classifier to unconstrained perturbations. Such a nearest neighbour strategy is useful when the in-distribution robustness is much larger than the unconstrained robustness, and permits the latter to match the former. This approach has recently been found to be successful in increasing the robustness of classifiers when accurate generative models can be learned in \cite{defensegan}. Other techniques \cite{ilyas2017robust} build on this approach, and further use methods to  increase the in-distribution robustness.
	
\subsection{Transferability of perturbations}

One of the most intriguing properties about adversarial perturbations is their transferability \cite{szegedy2013intriguing, liu2016delving} across different models. Under our data model distribution, we study the existence of transferable adversarial perturbations, and show that two models with approximately zero risk will have shared adversarial perturbations. 
\begin{theorem}[Transferability of perturbations]
	\label{thm:transferability}
	Let $f,h$ be two classifiers. Assume that $\mathbb{P} (f \circ g(z) \neq h \circ g(z) ) \leq \delta$ (e.g., if $f$ and $h$ have a risk bounded by $\delta/2$ for the data set generated by $g$). In addition, assume that 
	$\mathbb{P}(C_i(f)) + \delta \leq \frac{1}{2}$ for all $i$.\footnote{This assumption is only to simplify the statement, a general statement can be easily derived in the same way.} Then, 
	\begin{equation}
	\begin{aligned}
	& \mathbb{P}  \left\{ \exists v  : \| v \|_2 \leq \eps \text{ and } \begin{array}{ll} f(g(z)+v) \neq f(g(z)) \\ h(g(z)+v) \neq h(g(z)) \end{array} \right\}\\
	& \qquad \qquad \qquad \geq 1 - \sqrt{\frac{\pi}{2}} e^{-\omega^{-1}(\eps)^2/2} - 2 \delta.
	\end{aligned}
	\end{equation}
\end{theorem}

Compared to Theorem \ref{thm:image_space_bounds} which bounds the robustness to adversarial perturbations, the extra price to pay here to find \textit{transferable} adversarial perturbations is the $2 \delta$ term, which is small if the risk of both classifiers is small. Hence, our bounds provide a theoretical explanation for the existence of transferable adversarial perturbations, which were previously shown to exist in \cite{szegedy2013intriguing, liu2016delving}. 
The existence of transferable adversarial perturbations across several models with small risk has important security implications, as adversaries can, in principle, fool different classifiers with a single, classifier-agnostic, perturbation. The existence of such perturbations significantly reduces the difficulty of attacking (potentially black box) machine learning models.
	
	\subsection{Approximate generative model}

    In the previous results, we have assumed that the data distribution is exactly described by the generative model $g$ (i.e., $\mu = g_{*} (\nu)$ where $g_*(\nu)$ is the pushforward of $\nu$ via $g$). However, in many cases, such generative models only provide an \textit{approximation} to the true data distribution $\mu$. In this section, we specifically assume that the generated distribution $g_{*} (\nu)$ provides an approximation to the true underlying distribution in the 1-Wasserstein sense on the metric space $(\mathcal{X}, \| \cdot \|)$; i.e.,  $W(g_{*}(\nu), \mu) \leq \delta$, and derive upper bounds on the robustness. This assumption is in line with recent advances in generative models, whereby the generator provides a good approximation (in the Wasserstein sense) to the true distribution, but does not exactly fit it \cite{arjovsky2017wasserstein}. 
     We show here that similar upper bounds on the robustness (in expectation) hold, as long as $g_{*} (\nu)$ provides an accurate approximation of the true distribution $\mu$. 

	\begin{theorem}
		\label{thm:expectation_image_in}
		We use the same notations as in Theorem \ref{thm:image_space_bounds}. 
		Assume that the generator $g$ provides a $\delta$ \textit{approximation} of the true distribution $\mu$ in the 1-Wasserstein sense on the metric space $(\mathcal{X}, \| \cdot \|)$; that is, $W(g_*(\nu), \mu) \leq \delta$ (where $g_*(\nu)$ is the pushforward of $\nu$ via $g$), the following inequality holds provided $\omega$ is concave
		\[
		\ex{x \sim \mu}{r_{\text{unc}} (x)} \leq \omega \left( \sum_{i=1}^{K} -a_{\neq i} \Phi(-a_{\neq i}) + \frac{e^{-a_{\neq i}^2/2}}{\sqrt{2\pi}} \right) + \delta,
		\]
		where $r_{\text{unc}} (x)$ is the unconstrained robustness in the image space. In particular, for $K \geq 5$ equiprobable classes, we have
		\[
		\ex{x \sim \mu}{r_{\text{unc}} (x)} \leq \omega \left( \frac{\log(4\pi \log(K))}{\sqrt{2\log(K)}} \right) + \delta.
		\]
	\end{theorem}
	
In words, when the data is defined according to a distribution which can be \textit{approximated} by a smooth, high-dimensional generative model, our results show that arbitrary classifiers will have small adversarial examples in expectation. We also note that as $K$ grows, this bound decreases and even goes to zero under the sole condition that $\omega$ is continuous at $0$. Note however that the decrease is slow as it is only logarithmic.

	\section{Experimental evaluation}

	We now evaluate our bounds on the SVHN dataset \cite{netzer2011reading} which contains color images of house numbers, and the task is to classify the digit at the center of the image. In all this section, computations of perturbations are done using the algorithm in \cite{moosavi2015deepfool}.\footnote{Note that in order to estimate robustness quantities (e.g., $r_{\text{in}}$), we do not need the ground truth label, as the definition only involves the change of the estimated label. Estimation of the robustness can therefore be readily done for automatically generated images.} The dataset contains $73,257$ training images, and $26,032$ test images (we do not use the images in the 'extra' set). We train a DCGAN \cite{radford2015unsupervised} generative model on this dataset, with a latent vector dimension $d = 100$, and further consider several neural networks architectures for classification.\footnote{For the SVHN and CIFAR-10 experiments, we show examples of generated images and perturbed images in the appendix (Section C.3). Moreover, we provide in C.1 details on the architectures of the used models.} 
    For each classifier, the empirical robustness is compared to our upper bound.\footnote{To evaluate numerically the upper bound, we have used a probabilistic version of the modulus of continuity, where the property is not required to be satisfied for \textit{all} $z,z'$, but rather with high probability, and accounted for the error probability in the bound. We refer to the appendix for the detailed optimization used to estimate the smoothness parameters.} 
	 In addition to reporting the in-distribution and unconstrained robustness, we also report the robustness in the latent space: $r_Z = \min_{r} \| r \|_2 \text{ s.t. } f(g(z+r)) \neq f(g(z))$. For this robustness setting, note that the upper bound exactly corresponds to Theorem \ref{thm:image_space_bounds} with $\omega$ set to the identity map. Results are reported in Table \ref{tab:svhn_exp}. 
	
		\begin{table*}[t]
		\centering \renewcommand\cellgape{\Gape[2pt]}
		\begin{tabular}{l|c|c|c|c}
			\toprule
			& \makecell[c]{Upper bound\\ on robustness} & 2-Layer LeNet & ResNet-18 & ResNet-101 \\ 
			\midrule
			Error rate & - & 11\% & 4.8\% & 4.2 \% \\ \hline
			Robustness in the $\mathcal{Z}$-space & $16 \times 10^{-3}$ & $6.1 \times 10^{-3}$ & $6.1 \times 10^{-3}$ & $6.6 \times 10^{-3}$ \\ \hline
			\makecell[l]{In-distribution robustness} & $36 \times 10^{-2}$ & $3.3 \times 10^{-2}$ & $3.1 \times 10^{-2}$ & $3.1 \times 10^{-2}$  \\
			\makecell[l]{Unconstrained robustness} & $36 \times 10^{-2}$ & $0.39 \times 10^{-2}$ & $1.1 \times 10^{-2}$ & $1.4 \times 10^{-2}$ \\ \bottomrule
		\end{tabular}
		\caption{\label{tab:svhn_exp} Experiments on SVHN dataset. We report the $25\%$ percentile of the \textit{normalized} robustness at each cell, where probabilities are computed either theoretically (for the upper bound) or empirically. 
		More precisely, we report the following quantities for the upper bound column: For the \textbf{robustness in the $\mathcal{Z}$ space}, we report $t / \mathbb{E} (\| z \|_2)$ such that $\mathbb{P} \left( \min_{r} \| r \|_2 \text{ s.t. } f(g(z+r)) \neq f(g(z)) \leq t \right) \geq 0.25$, using Theorem \ref{thm:image_space_bounds} with $\omega$ taken as identity. For the \textbf{robustness in image-space}, we report $t / \mathbb{E} (\| g(z) \|_2)$ such that $\mathbb{P} \left( r_{\text{in}} (x) \leq t \right) \geq 0.25$, using Theorem \ref{thm:image_space_bounds}, with $\omega$ estimated empirically (Section C.2 in appendix).}
	\end{table*}
	
	Observe first that the upper bound on the robustness in the latent space is of the same order of magnitude as the empirical robustness computed in the $\mathcal{Z}$-space, for the different tested classifiers. This suggests that the isoperimetric inequality (which is the only source of inequality in our bound, when factoring out smoothness) provides a reasonable baseline that is on par with the robustness of best classifiers. 
	In the image space, the theoretical prediction from our classifier-agnostic bounds is one order of magnitude larger than the empirical estimates. Note however that our bound is still non-vacuous, as it predicts the norm of the required perturbation to be approximately $1/3$ of the norm of images (i.e., normalized robustness of $0.36$). 
	This potentially leaves room for improving the robustness in the image space. Moreover, we believe that the bound on the robustness in the image space is not tight (unlike the bound in the $\mathcal{Z}$ space) as the smoothness assumption on $g$ can be conservative.
	
	Further comparisons of the figures between in-distribution and unconstrained robustness in the image space interestingly show that for the simple LeNet architecture, a large gap exists between these two quantities. However, by using more complex classifiers (ResNet-18 and ResNet-101), the gap between in-distribution and unconstrained robustness gets smaller. Recall that Theorem~\ref{thm:relation_two_robustness} says that any classifier can be modified in a way that the in-distribution robustness and unconstrained robustness only differ by a factor $2$, while preserving the accuracy. But this modification may result in a more complicated classifier compared to the original one; for example starting with a linear classifier, the modified classifier will in general not be linear. This interestingly matches with our numerical values for this experiment, as the multiplicative gap between in-distribution and unconstrained robustness approaches $2$ as we make the classification function more complex (e.g., in-distribution robustness of $3.1 \times 10^{-2}$ and out-distribution $1.4 \times 10^{-2}$ for ResNet-101).

    We now consider the more complex CIFAR-10 dataset \cite{krizhevsky2009learning}. The CIFAR-10 dataset consists of 10 classes of $32 \times 32$ color natural images. Similarly to the previous experiment, we used a DCGAN generative model with $d = 100$, and tested the robustness of state-of-the-art deep neural network classifiers. 
    Quantitative results are reported in Table \ref{tab:cifar_exp}. Our bounds notably predict that any classifier defined on this task will have perturbations not exceeding $1/10$ of the norm of the image, for $25\%$ of the datapoints in the distribution. Note that using the PGD adversarial training strategy of \cite{madry2017towards} (which constitutes one of the most robust models to date \cite{uesato2018adversarial}), the robustness is significantly improved, despite still being $\sim 1$ order of magnitude smaller than the baseline of $0.1$ for the in-distribution robustness. The construction of more robust classifiers, alongside better empirical estimates of the quantities involved in the bound/improved bounds will hopefully lead to a convergence of these two quantities, hence guaranteeing optimality of the robustness of our classifiers.
    
    \begin{table*}[t]
    	\centering \renewcommand\cellgape{\Gape[2pt]}
    	\begin{tabular}{l|c|c|c|c}
    		\toprule
    		& \makecell[c]{Upper bound\\on robustness} & VGG \cite{simonyan2014very} & Wide ResNet \cite{zagoruyko2016wide} & \makecell[c]{Wide ResNet \\ + Adv. training \\ \cite{madry2017towards, uesato2018adversarial}} \\ \midrule
    		Error rate & - & $5.5 \%$ &  $3.9 \%$ & $16.0 \%$ \\ \hline
    		Robustness in the $\mathcal{Z}$-space & $0.016$  & $2.5 \times 10^{-3}$ &  $3.0 \times 10^{-3}$ & $3.6 \times 10^{-3}$ \\ \hline 
    		\makecell[l]{In-distribution robustness} & $0.10$ & $4.8 \times 10^{-3}$ & $5.9 \times 10^{-3}$ & $8.3 \times 10^{-3}$ \\ 
    		\makecell[l]{Unconstrained robustness}  & $0.10$ & $0.23 \times 10^{-3}$  & $0.20 \times 10^{-3}$ & $2.0 \times 10^{-3}$\\
    		\bottomrule
    		
    	\end{tabular}
    	\caption{\label{tab:cifar_exp} Experiments on CIFAR-10 (same setting as in Table \ref{tab:svhn_exp}). See appendix for details about models.}
    \end{table*}

	\section{Discussion}

    We have shown the existence of a baseline robustness that no classifier can surpass, whenever the distribution is approximable by a generative model mapping latent representations to images. 
    The bounds lead to informative numerical results: for example, on the CIFAR-10 task (with a DCGAN approximator), our upper bound shows that a significant portion of datapoints can be fooled with a perturbation of magnitude $10\%$ that of an image. Existing classifiers however do not match the derived upper bound. Moving forward, we expect the design of more robust classifiers to get closer to this upper bound. 
The existence of a baseline robustness is fundamental in that context in order to measure the progress made and compare to the optimal robustness we can hope to achieve. 

    In addition to providing a baseline, this work has several practical implications on the \textit{robustness} front. To construct classifiers with better robustness, our analysis suggests that these should have linear decision boundaries in the \textit{latent} space; in particular, classifiers with multiple disconnected classification regions will be more prone to small perturbations. We further provided a constructive way to provably close the gap between unconstrained robustness and in-distribution robustness. 
    
    Our analysis at the intersection of classifiers' robustness and generative modeling has further led to insights onto \textit{generative models}, due to its intriguing generality. If we take as a premise that human visual system classifiers require large-norm perturbations to be fooled (which is implicitly assumed in many works on adversarial robustness, though see \cite{elsayed2018adversarial}), our work shows that natural image distributions cannot be modeled as very high dimensional and smooth mappings. While current dimensions used for the latent space (e.g., $d = 100$) do not lead to any contradiction with this assumption (as upper bounds are sufficiently large), moving to higher dimensions for more complex datasets might lead to very small bounds. To model such datasets, the prior distribution, smoothness and dimension properties should therefore be carefully set to avoid contradictions with the premise. For example, conditional generative models can be seen as non-smooth generative models, as different generating functions are used for each class. We finally note that the derived results do  bound the \textit{norm} of the perturbation, and not the human perceptibility, which is much harder to quantify. We leave it as an open question to derive bounds on more perceptual metrics.
    
    \subsubsection*{Acknowledgments}

    A.F. would like thank Seyed Moosavi, Wojtek Czarnecki, Neil Rabinowitz, Bernardino Romera-Paredes and the DeepMind team for useful feedbacks and discussions.

	\appendix
	
	\section{Proofs}
	
	\label{sec:proofs}
	
	\subsection{Useful results}
	
	Recall that we write the cumulative distribution function for the standard Gaussian distribution $\Phi(x) = \frac{1}{\sqrt{2\pi}} \int_{-\infty}^x e^{-u^2/2} du$. We state the Gaussian isoperimetric inequality~\cite{borell1975brunn,sudakov1978extremal}, the main technical tool used in to prove the results in this paper. 
	\begin{theorem}[Gaussian isoperimetric inequality]
	\label{th:gaussian_isoperimetry}
	Let $\nu_{d}$ be the Gaussian measure on $\RR^d$. Let $A \subseteq \RR^d$ and let $A_{\eta} = \{z \in \RR^d : \exists z' \in A \text{ s.t. } \|z-z'\|_2 \leq \eta\}$. If $\nu_d(A) = \Phi(a)$ then $\nu_d(A_{\eta}) \geq \Phi(a+\eta)$.
	\end{theorem}
	
	We then state some useful bounds on the cumulative distribution function for the Gaussian distribution $\Phi$.

	\begin{lemma}[see e.g.,~\cite{duembgen2010bounding}]
	We have for $x \geq 0$, 
	\begin{align}
	\label{eq:bounds_Phi}
	1 - \frac{e^{-x^2/2}}{\sqrt{2\pi}} \frac{2}{x + \sqrt{x^2+8/\pi}}  \leq \Phi(x) 
	\leq 1 - \frac{e^{-x^2/2}}{\sqrt{2\pi}}  \frac{2}{x + \sqrt{x^2+4}} \ .
	\end{align}
	\end{lemma}
	
	\begin{lemma}
	\label{lem:prop_gaussian_cdf}
	Let $p \in [1/2,1]$, we have for all $\eps > 0$,
	\begin{align}
	\label{eq:bound_Phi_plus_eps}
	\Phi( \Phi^{-1}(p) + \eps) \geq 1 - (1-p) \sqrt{\frac{\pi}{2}} e^{-\eps^2/2} e^{-\eps \Phi^{-1}(p) }   \ .
	\end{align}
	If $p = 1 - \frac{1}{K}$ for $K \geq 5$ and $\eps \geq 1$, we have
	\begin{align}
	\label{eq:bound_Phi_explicit_K}
	\Phi( \Phi^{-1}(1 - \frac{1}{K}) + \eps) \geq 1 - \frac{1}{K} \sqrt{\frac{\pi}{2}} e^{-\eps^2/2} e^{-\eps \sqrt{\log\left(\frac{K^2}{4\pi \log(K)}\right)}} \ .
	\end{align}
	\end{lemma}
	\begin{proof}
	As $p \geq 1/2$, we have $\Phi^{-1}(p) \geq 0$. Thus,
	\begin{align*}
	\Phi( \Phi^{-1}(p) + \eps) 
	&\geq 1 - \frac{1}{\sqrt{2\pi}} \frac{2 e^{-(\Phi^{-1}(p) + \eps)^2/2}}{\Phi^{-1}(p) + \eps + \sqrt{(\Phi^{-1}(p) + \eps)^2 + 8/\pi}} \\
        &= 1 - \frac{1}{\sqrt{2\pi}} \frac{2 e^{-\Phi^{-1}(p)^2/2 - \eps^2/2 - \eps \Phi^{-1}(p)}}{\Phi^{-1}(p) + \eps + \sqrt{(\Phi^{-1}(p) + \eps)^2 + 8/\pi}} \\
        &= 1 - \left(\frac{1}{\sqrt{2\pi}} \frac{2e^{-\Phi^{-1}(p)^2/2}}{\Phi^{-1}(p) + \sqrt{\Phi^{-1}(p)^2 + 4}} \right) \\
        & \times \frac{\Phi^{-1}(p) + \sqrt{\Phi^{-1}(p)^2 + 4}}{\Phi^{-1}(p) + \eps + \sqrt{(\Phi^{-1}(p) + \eps)^2 + 8/\pi}} e^{- \eps^2/2 - \eps \Phi^{-1}(p)} \ .
        \end{align*}
        Now we use the fact that 
        \begin{align*}
        \left(\frac{1}{\sqrt{2\pi}} \frac{2e^{-\Phi^{-1}(p)^2/2}}{\Phi^{-1}(p) + \sqrt{\Phi^{-1}(p)^2 + 4}} \right) \leq 1 - \Phi(\Phi^{-1}(p)) = 1 - p \ .
        \end{align*}
        As a result,
        \begin{align*}     
        & \Phi( \Phi^{-1}(p) + \eps)  \\
        & \geq 1 - (1 - p) e^{- \eps^2/2 - \eps \Phi^{-1}(p)}  \frac{\Phi^{-1}(p) + \sqrt{\Phi^{-1}(p)^2 + 4}}{\Phi^{-1}(p) + \eps + \sqrt{(\Phi^{-1}(p) + \eps)^2 + 8/\pi}} \\
        &\geq 1 - (1 - p) e^{- \eps^2/2 - \eps \Phi^{-1}(p)}  \frac{\Phi^{-1}(p) + \sqrt{\Phi^{-1}(p)^2 + 4}}{\Phi^{-1}(p) + \sqrt{\Phi^{-1}(p)^2 + 8/\pi}} \\
        &\geq 1 - (1 - p) e^{ - \eps^2/2 }  e^{-\eps \Phi^{-1}(p) }  \frac{\sqrt{4}}{\sqrt{8/\pi}} \ .
	\end{align*}

In the case $p = 1 - \frac{1}{K}$, it suffices to show that that for $K \geq 5$, we have 
\begin{equation}
\label{eq:lbphiinv}
\Phi^{-1}(1-1/K) \geq \sqrt{\log\left(\frac{K^2}{4\pi \log(K)}\right)} \ .
\end{equation}

Using the upper bound in~\eqref{eq:bounds_Phi}, and the fact that $x+\sqrt{x^2+2}\leq 2\sqrt{x^2+1}$, it suffices to show that $\frac{1}{2} \frac{e^{-x^2}}{\sqrt{\pi}\sqrt{x^2+1}} \geq \frac{1}{K}$ where $x = \sqrt{\frac{1}{2} \log\left(\frac{K^2}{4\pi \log(K)}\right)}$. This inequality is equivalent to showing that $\sqrt{\log(K)} \geq \sqrt{x^2+1}$ for the same value of $x$. If we let $u = \log(K)$ this amounts to showing that $\sqrt{u} \geq \sqrt{u - \frac{1}{2} \log(4\pi u) + 1}$ for all $u \geq \log(5)$. For such $u$ one can verify that $- \frac{1}{2} \log(4\pi u) + 1 \leq 0$ and so clearly the inequality is satisfied.

	\end{proof}
	
	\subsection{Proof of Theorem \ref{thm:image_space_bounds}}    
	\label{sec:proof_main_theorem}
	
	\begin{proof}
		To prove the general bound in Eq.~\eqref{eq:main_theorem_image_space}, we define
		\[
		C_{i\rightarrow} = \{x \in C_i : \dist(x,\cup_{j \neq i} C_j) \leq \eta\}.
		\]
		Here, $\dist(x,C)$ is defined as $\inf_{x' \in C} \| x - x' \|$. Let us also introduce the following sets in the $z$-space: $B_i = g^{-1}(C_i)$ and $B_{i\rightarrow} = \{z \in B_i : \dist(z,\cup_{j \neq i} B_j) \leq \omega^{-1}(\eta)\}$. It is easy to verify that $g(B_{i\rightarrow}) \subseteq C_{i\rightarrow}$. Thus we have $\PP(C_{i\rightarrow}) = \nu(g^{-1}(C_{i\rightarrow})) \geq \nu(B_{i\rightarrow})$. Now note that $B_{i\rightarrow} \bigcup \cup_{j\neq i} B_{j}$ is nothing but the set of points that are at distance at most $\omega^{-1}(\eps)$ from $\cup_{j \neq i} B_j$. As such, by the Gaussian isoperimetric inequality (Theorem~\ref{th:gaussian_isoperimetry}) applied with $A = \cup_{j \neq i} B_j$ and $a = a_{\neq i}$, we have $\nu(B_{i\rightarrow}(\eps)) + \nu(\cup_{j \neq i} B_j) \geq \Phi(a_{\neq i}+\omega^{-1}(\eps))$, i.e., $\nu(B_{i\rightarrow}) \geq \Phi(a_{\neq i}+\omega^{-1}(\eps)) - \Phi(a_{\neq i})$. As $B_{i\rightarrow}$ are disjoint for different $i$, we have
		\begin{align*}
		\nu(\cup_i B_{i\rightarrow}(\eta)) \geq \sum_{i=1}^{K} (\Phi(a_{\neq i} + \omega^{-1}(\eta)) - \Phi(a_{\neq i})) \ .
		\end{align*}
The proof of inequality~\eqref{eq:main_theorem_image_space} of the main text then follows by using $\PP(C_{i\rightarrow}) \geq \nu(B_{i\rightarrow})$.
		
		To prove inequality~\eqref{eq:prob_bound_half}, observe that if $\PP(C_i) \leq \frac{1}{2}$ for all $i$, then $\PP(\cup_{j \neq i} C_j) \geq \frac{1}{2}$ for all $i$. Then we use the bound~\eqref{eq:bound_Phi_plus_eps} to get,		
		\begin{align*}
		\PP(\cup_i C_{i\rightarrow}(\eta)) 
		&\geq \sum_{i=1}^{K} (\Phi(\Phi^{-1}(\PP(\cup_{j \neq i} C_j)) + \eta) - \PP(\cup_{j \neq i} C_j)) \\
		&\geq \sum_{i=1}^{K} (1 - (1 - \PP(\cup_{j \neq i} C_j)) \sqrt{\frac{\pi}{2}} e^{-\eps^2/2} - \PP(\cup_{j \neq i} C_j)) \\
		&= (1 - \sqrt{\frac{\pi}{2}} e^{-\eps^2/2}) \sum_{i=1}^{K} (1 - \PP(\cup_{j \neq i} C_j)) \\
		&= 1 - \sqrt{\frac{\pi}{2}} e^{-\eps^2/2} \ .
		\end{align*}		
		
		For the bound~\eqref{eq:prob_bound_class_dependence} that makes explicit the dependence on the number of classes, we simply use the more explicit bound in~\eqref{eq:bound_Phi_explicit_K}.
	\end{proof}

	\subsection{Proof of Theorem \ref{thm:relation_two_robustness}}
	
	\begin{proof}
	Let $x = g(z) \in \mathcal{X}$ and $x' \in \mathcal{X}$. 
	Let $z^*$ be such that $\tilde{f}(x') = f(g(z^*))$. By definition of $\tilde{f}$, we have $\| x' - g(z^*) \| \leq \| x' - g(z) \|$. As such, using the triangle inequality, we get
	\begin{align*}
	\| g(z) - g(z^*) \| &\leq \| g(z) - x' \| + \| x' - g(z^*) \| \\
	&\leq 2 \| g(z) - x' \| \ .
	\end{align*}
	Taking the minimum over all $x'$ such that $\tilde{f}(x) \neq \tilde{f}(x')$, we obtain
	\begin{align*}
	r_{\text{in}} (x) \leq 2 r_{\text{unc}}(x).
	\end{align*}
    \end{proof}

	\subsection{Proof of Theorem \ref{thm:transferability}}
	
	\begin{proof}
    We use the same notations as in the proof of Theorem \ref{thm:image_space_bounds}: let $B_i(f) = g^{-1}(C_i(f))$ and $B_i(h) = g^{-1}(C_i(h))$, and let
    \[
    B_{i\rightarrow} = \{z \in B_i(f) \cup B_i(h) : \dist(x, \overline{B_i(f)} \cap \overline{B_i(h)} ) \leq \omega^{-1}(\eta)\}.
    \]
    where the notation $\overline{B}$ stands for the complement of $B$.
    
		Note that $B_i(f) \cup B_i(h) = \overline{ \overline{B_i(f)} \cap \overline{B_i(h)}}$.
		We have $\nu(\overline{B_i(f)} \cap \overline{B_i(h)} ) \geq \nu(\overline{B_i(f)} ) - \delta = 1 - \nu(B_i(f)) - \delta \geq \frac{1}{2}$. Thus, using the Gaussian isoperimetric inequality with $A = \overline{B_i(f)} \cap \overline{B_i(h)}$, we obtain
		\begin{align*}
		\nu(B_{i\rightarrow}) + \nu(\overline{B_i(f)}\cap \overline{B_i(h)})
		\geq 1- \left(1 - \nu(\overline{B_i(f)} \cap \overline{B_i(h)})\right) \sqrt{\frac{\pi}{2}} e^{-\eps^2/2},
		\end{align*}
		where we also used inequality~\eqref{eq:bound_Phi_plus_eps}. As a result,
		\begin{align*}
		\nu(B_{i\rightarrow}) &\geq (1 - \nu(\overline{B_i(f)} \cap \overline{B_i(h)}) (1 - \sqrt{\frac{\pi}{2}}e^{-\eps^2/2}) \\
		&\geq \nu(B_i(f)) (1 - \sqrt{\frac{\pi}{2}} e^{-\eps^2/2}) \ .
		\end{align*}
		Now assume that $z \in B_{i\rightarrow}$ but also $z \in B_i(f) \cap B_i(h)$. Then it is classified as $i$ for both $f$ and $h$. In addition, the condition $z \in B_{i\rightarrow}$ ensures that there exists $z' \in \overline{B_i(f)} \cap \overline{B_i(h)}$ such that $\| z - z' \|_2 \leq \omega^{-1}(\eps)$. Setting $v = g(z') - g(z)$, we have that $f(g(z)+v) \neq f(g(z))$ and $h(g(z)+v) \neq h(g(z))$ and $\|v\| \leq \omega(\|z-z'\|) \leq \eps$. As such it suffices to show that the set $B_{i\rightarrow} \cap (B_i(f) \cap B_i(h))$ has sufficiently large measure. Indeed, we have
		\begin{align*}
		& \nu(B_{i\rightarrow} \cap (B_i(f) \cap B_i(h)) ) \\
		&\geq \nu(B_{i\rightarrow}) - \nu( B_i(f) \cap \overline{B_i(h)}) - \nu(\overline{B_i(f)} \cap B_i(h)) \ .
		\end{align*}
		Summing over $i$, we get
		\begin{align*}
		\sum_{i=1}^K \nu(B_{i\rightarrow} \cap (B_i(f) \cap B_i(h)) ) 
		&\geq 1 - \sqrt{\frac{\pi}{2}} e^{-\eps^2/2} - 2\delta \ ,
		\end{align*}
		because $\sum_{i=1}^K \nu(B_i(f) \cap \overline{B_i(h)}) + \nu(\overline{B_i(f)} \cap B_i(h)) = 2 \cdot \nu\left\{f \circ g(z) \neq h \circ g(z)\right\} \leq 2\delta$.
		
	\end{proof}

		\subsection{Proof of Theorem \ref{thm:expectation_image_in}}
	
	\begin{proof}
We first treat the case $\delta=0$. Given $z$ we denote by $r_{\Z}(z) = \min \left\{\|r\|_2 : f(g(z+r)) \neq f(g(z))\right\}$. Then it is easy to see that $r_{\text{in}}(g(z)) \leq \omega(r_{\Z}(z))$. As such we have $\E_x[r_{\text{in}} (x)] = \E_z[r_{\text{in}}(g(z))] \leq \E_z[\omega(r_{\Z}(z))] \leq \omega(\E_z[r_{\Z}(z)])$. Now we have
\[
\E_z [r_{\Z} (z) ] = \int_{0}^{\infty} \mathbb{P}_z[ r_{\Z} (z) \geq \eps ] d\eps.
\]
Using a bound similar to Theorem~\ref{thm:image_space_bounds} applied to $r_{\Z}$ we get
\[
\begin{aligned}
\E_z [r_{\Z} (z) ] &\leq \int_{0}^{\infty} 
\left(1 - \sum_{i=1}^{K} \Phi(a_{\neq i} + \eps) - \Phi(a_{\neq i})\right)  d\eps\\
&= \sum_{i=1}^{K}  \int_{0}^{\infty} \Phi(-a_{\neq i} - \eps)  d\eps\\
\end{aligned}
\]
where in the equality, we used the fact that $1 = \sum_{i=1}^K (1-\Phi(a_{\neq i}))$. Now observe that for any $a \in \mathbb{R}$,
\begin{align*}
\int_{0}^{\infty} \Phi(-a - \eps) d\eps
&= \int_{a}^{\infty} \int_{-\infty}^{-u} \frac{e^{-t^2/2}}{\sqrt{2\pi}} dt du \\
&= \int_{-\infty}^{\infty} \left( \int_{a}^{\infty} \mathbf{1}_{t \leq -u} du \right)  \frac{e^{-t^2/2}}{\sqrt{2\pi}} dt \\
&= \int_{-\infty}^{\infty} (-t-a)\mathbf{1}_{a \leq -t} \frac{e^{-t^2/2}}{\sqrt{2\pi}} dt \\
&= \frac{e^{-a^2/2}}{\sqrt{2\pi}} -a \Phi(-a)  .
\end{align*}
As a result,
\[
\begin{aligned}
\E_z [r_{\Z} (z) ] 
&\leq \sum_{i=1}^{K} -a_{\neq i} \Phi(-a_{\neq i}) + \frac{e^{-a_{\neq i}^2/2}}{\sqrt{2\pi}}
\end{aligned}
\]
This establishes the first inequality.

Assuming now that the classes are equiprobable, i.e., $a_{\neq i} = \Phi^{-1}(1-1/K) =: a(K)$ for all $i$ we get that
\[
\E[r_{\text{in}}(x)] \leq \omega\left(-a(K)^2 + \frac{K}{\sqrt{2\pi}} e^{-a(K)^2/2}\right).
\]
Using the bound \eqref{eq:lbphiinv} on $a(K)$ we get:
\[
\begin{aligned}
\E[r_{\text{in}}(x)] &\leq \omega\left(\sqrt{2\log(K)} - \sqrt{2\log(K) - \log(4\pi\log(K))}\right)\\
     &= \omega\left(\frac{\log(4\pi\log(K))}{\sqrt{2\log(K)} + \sqrt{2\log(K) - \log(4\pi\log(K))}}\right)\\
     &\leq \omega\left(\frac{\log(4\pi\log(K))}{\sqrt{2\log(K)}}\right)
\end{aligned}
\]

	We assume now that $g$ is such that $W(g_*(\nu), \mu) \leq \delta$, where $W$ denotes the Wasserstein distance in $(\mathcal{X}, \| \cdot \|)$.  Let $(X, X')$ be a coupling with $X \sim \mu$ and $X' \sim g_*(\nu)$. We will construct a random variable $X''$ such that almost surely $X''$ and $X$ are classified differently. We define $X'' = X'$ if $X$ and $X'$ are classified differently and otherwise $X'' = X' + \vec{r}^*(X')$ where $\vec{r}^*(X')$ is defined to be a vector of minimum norm such that $X' + \vec{r}^*(X')$ and $X'$ are classified differently. Then we have
	\begin{align*}
	& \ex{x \sim \mu}{r_{\text{unc}}(x)} \\
	& \leq \mathbb{E} \|X - X''\| \\
	&=  \mathbb{E} (\mathbf{1}_{f(X) \neq f(X')} \| X - X' \|) + \mathbb{E} (\mathbf{1}_{f(X) = f(X')} \| X - (X' + \vec{r}^*(X')) \| ) \\
	&\leq  \mathbb{E}\| X - X' \|  + \mathbb{E} \| \vec{r}^*(X')) \|  \ .
	\end{align*}
	By choosing a coupling such that $W(g_*(\nu),\nu) = \mathbb{E}\| X - X' \|$, we get $\mathbb{E}\| X - X' \|  \leq \delta$. In addition, $\mathbb{E} \| \vec{r}^*(X')) \| \leq \mathbb{E}_{x} r_{\text{in}}(x)$. The statement therefore follows.
	\end{proof}

\section{Toy example: tightness of Theorem \ref{thm:image_space_bounds}}

\label{sec:complementary_remark_4}

As an illustration to Remark 4, we explicitly show through a toy example that a classifier which is not linear in the $\mathcal{Z}$-space can be significantly less robust than a linear one. 

\begin{figure}[t]
	\centering
	\begin{subfigure}[t]{0.45\textwidth}
		\includegraphics[width=1\textwidth]{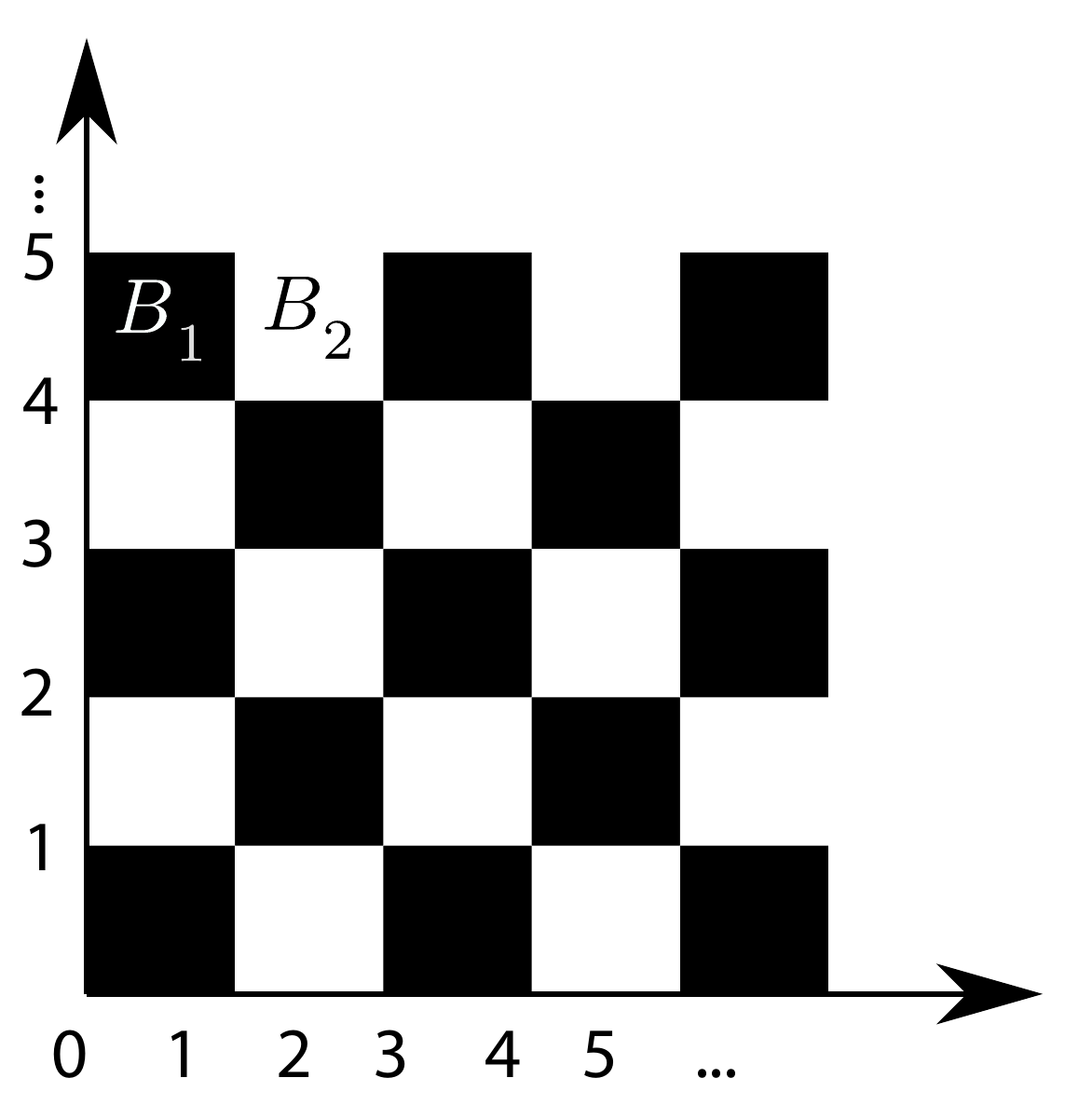}
		\caption{\label{fig:checkerboard_example}}
	\end{subfigure}
	\begin{subfigure}[t]{0.45\textwidth}
		\includegraphics[width=1\textwidth]{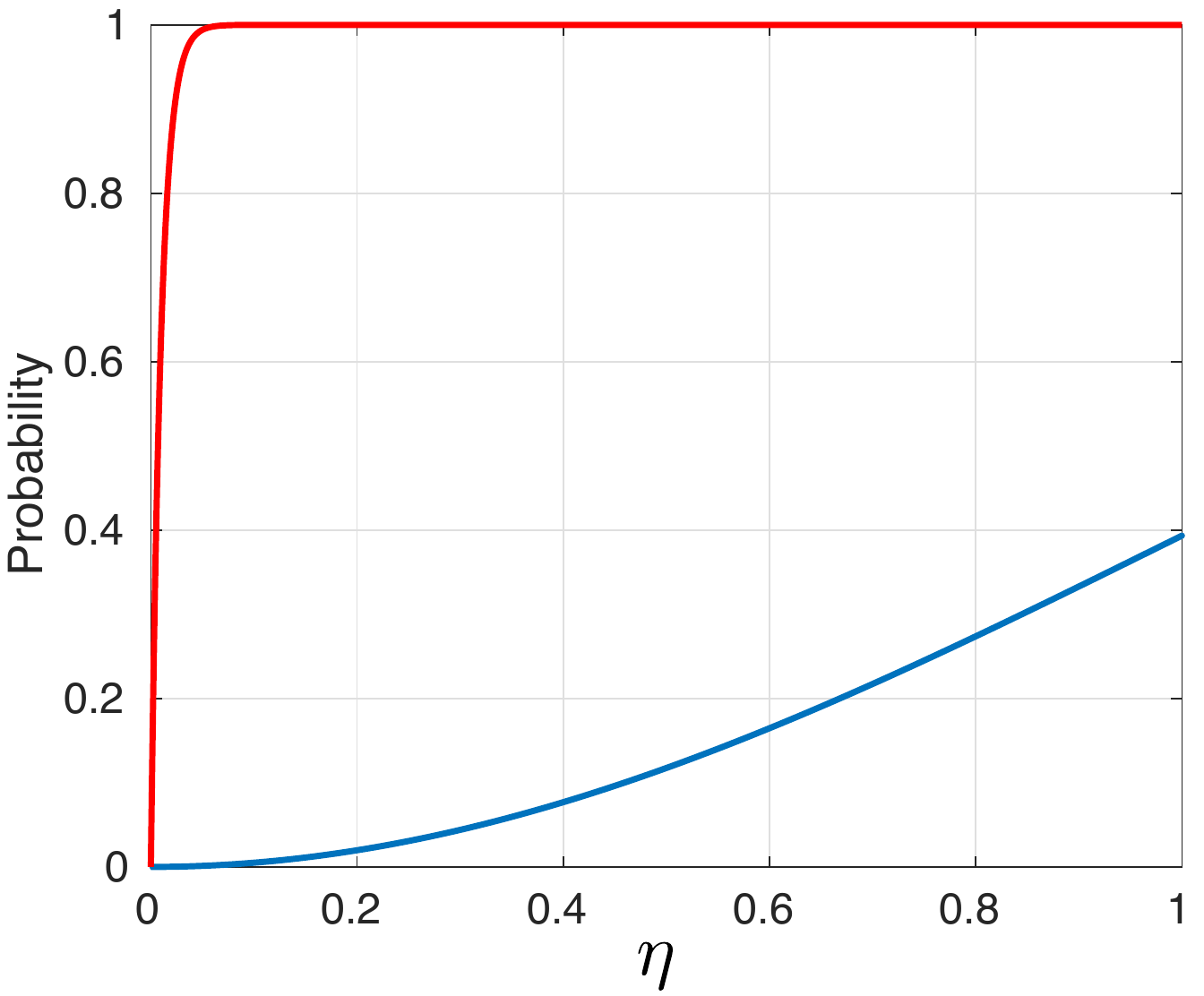}
		\caption{\label{fig:fig_plot_conv_rate}}
	\end{subfigure}
	\caption{Left: Illustration of checkerboard example. Right: Lower bound on robustness as a function of $\eta$ for the general result in Theorem \ref{thm:image_space_bounds} (blue curve) and the checkerboard example in Eq. \ref{eq:isoperimetry_checkerboard} (red curve).}
\end{figure}

\begin{example}[Checkerboard class partitions]
	\label{ex:checkerboard_class_partition}
	Assume that $B_1 = g^{-1}(C_1)$ and $B_2 = g^{-1}(C_2)$ are given by:
	\begin{itemize}
		\item $B_{1} = \{ (z_1, \dots, z_d) : \sum_{i=1}^d \floor{z_i} \mod 2 = 0 \}$,
		\item $B_2 = \RR^d - B_1$.
	\end{itemize}
	See Fig. \ref{fig:checkerboard_example} for an illustration. Then, we have
	\begin{align}
	\label{eq:isoperimetry_checkerboard}
	\mathbb{P}\left( z \in B_1 \text{ and } \dist(z,B_2) \leq \eta \right) + \mathbb{P}\left( z \in B_2 \text{ and } \dist(z,B_1) \leq \eta \right) \geq 1 - (1-\eta)^d.
	\end{align}
	Fig \ref{fig:fig_plot_conv_rate} compares the general bound in Theorem \ref{thm:image_space_bounds} to Eq. \eqref{eq:isoperimetry_checkerboard}. As can be seen, in the checkerboard partition example, the probability of fooling converges much quicker to $1$ (wrt $\eta$) than the general result in Theorem \ref{thm:image_space_bounds}. Hence, a classifier that creates many disconnected classification regions can be much more vulnerable to perturbations than a linear classifier in the latent space.
\end{example}
		\begin{proof}
	We have $\nu(B_1) = \nu(B_2) = \frac{1}{2}$. Let $z \in \RR^d$ in $B_2$ be such that for some $i \in \{1, \dots, d\}$, $z_i - \floor{z_i} \in [0, \eta) \cup (1-\eta, 1)$, then $z - \eta e_i \in B_1$ or $z + \eta e_i \in B_1$, and thus $z$ is at distance at most $\eta$ from $B_1$. As a result, if $z$ is at distance $> \eta$ from $B_1$, then for all $i \in \{1, \dots, d\}$, $z_i - \floor{z_i} \in [\eta, 1-\eta]$. As a result,
	\begin{align*}
	& \mathbb{P}_{z}(z \in B_2, \dist(z, B_1) > \eta) \\
	&\leq \mathbb{P}_{z}(z \in B_2, \forall i, z_i - \floor{z_i} \in [\eps, 1-\eps] ) \\
	&= \frac{1}{\sqrt{2\pi}^d}\sum_{\substack{(j_1, \dots, j_d) \in \ZZ^d,\\j_1 + \dots + j_d \mod 2 = 1}} \int_{j_1+\eta}^{j_1+1 - \eta} dz_1 \dots \int_{j_d+\eta}^{j_d + 1 - \eta} dz_d e^{-\frac{\sum_i z_i^2}{2}} \ .
	\end{align*}
	Now observe that for any $j \in \ZZ$, as the function $z \mapsto e^{-z^2/2}$ is monotone on the interval $[j, j+1]$ (nondecreasing if $j < 0$ and nonincreasing if $j\geq 0$). Thus, we have $\int_{j+\eta}^{j+1 - \eta} e^{-\frac{z^2}{2}} dz \leq (1-\eta) \int_{j}^{j+1} e^{-\frac{z^2}{2}} dz$, when $\eta \leq \frac{1}{2}$. As a result,
	\begin{align*}
	& \mathbb{P}_{z}(z \in B_2, \dist(z, B_1) > \eta) \\
	&\leq \frac{1}{\sqrt{2\pi}^d} (1-\eta)^d \sum_{\substack{(j_1, \dots, j_d) \in \ZZ^d,\\\sum_i j_i \mod 2 = 1}} \int_{j_1}^{j_1+1} d z_1 \dots \int_{j_d}^{j_d + 1} d z_d e^{-\frac{\sum_i z_i^2}{2}} \\
	&= (1-\eta)^d \mathbb{P}_{z}(z \in B_2) \\
	&= \frac{1}{2} (1-\eta)^{d} \ .
	\end{align*}
	With the same reasoning, $\mathbb{P}_{z}(z \in B_1, \dist(z, B_2) > \eta) \leq \frac{1}{2} (1-\eta)^d $ and gives inequality~\eqref{eq:isoperimetry_checkerboard}.
\end{proof}

\section{Experimental results}
	\label{sec:additional_exps}

    \subsection{Details of the used models}
	
	For the SVHN dataset, we resize the images to $64 \times 64$. For the generative model, we use the PyTorch implementation of DCGAN available on \url{https://github.com/pytorch/examples/blob/master/dcgan/main.py} using the default parameters for architecture and optimization. The $2$-layer LeNet classifier has the following architecture:
	\begin{align*}
	& \text{Conv} (5, 2, 16) \rightarrow \text{ReLU} \rightarrow \text{MaxPool}(4) \\
	& \rightarrow \text{Conv} (5, 2, 32) \rightarrow \text{ReLU} \rightarrow \text{MaxPool}(4) \rightarrow \text{FC} (10),
	\end{align*}
	where the parameters of $\text{Conv}$ are kernel size, padding and number of filters, respectively. We used the ResNet18 and ResNet101 architectures available on  \url{https://github.com/kuangliu/pytorch-cifar/blob/master/models/resnet.py}, with a kernel size of $5$ for $\text{Conv1}$ and a stride of $2$. For all $3$ architectures, we used SGD with a learning rate of $0.01$, momentum of $0.9$, batch size of $100$. To solve the problem in Eq. \ref{eq:omega_eta}, we use gradient descent (for the maximization of $\| g(z) - g(z') \|_2$) with learning rate $0.1$ for $1,000$ steps. The upper bound was computed based on $100$ samples of $z$. 
	
	For the CIFAR-10 experiment, we use a similar DCGAN generative model. The VGG-type architecture has 11 conv layers, each of kernel size $3$, with number of output channels $(64, 64, 128, 128, 128, 256, 256, 256, 512, 512, 512)$ and stride $(1, 1, 2, 1, 1, 2, 1, 1, 2, 1, 1)$. Each conv layer is followed by BatchNorm and a ReLU function. For the WideResNet architecture, we use the WRN-28-10 model available on \url{https://github.com/szagoruyko/wide-residual-networks}. SGD is used with learning rate $0.1$, momentum $0.9$, and batchsize $100$. For the adversarially trained Wide ResNet with PGD training, we have used the model of \cite{madry2017towards}.

    \subsection{Numerical evaluation of the upper bound}
    
    To evaluate numerically the upper bound, we have used a probabilistic version of the modulus of continuity, where the property is not required to be satisfied for \textit{all} $z,z'$, but rather with high probability, and accounted for the error probability in the bound. Specifically, while the modulus of continuity function is given by
	$\omega(\delta) = \max_{z} \max_{z': \| z - z' \|_2 \leq \delta} \| g(z) - g(z') \|_2$, we use in the experiments a probabilistic version of the modulus of continuity, given by:
	\begin{equation}
	\label{eq:omega_eta}
	\omega_{\kappa}(\delta) = \min \left\{ \alpha :  \mathbb{P} \left( \sup_{z': \| z - z' \|_2 \leq \delta} \| g(z) - g(z') \|_2 \geq \alpha \right) \leq \kappa \right\}.
	\end{equation}
	Then, the following bound holds for any $\delta, \kappa$:
	\begin{equation}
	\label{eq:new_upper_bound}
	\mathbb{P} \left( r_{\text{in}} (x) \geq \omega_{\kappa} (\delta) \right) \leq \kappa + \underbrace{\mathbb{P} \left( \exists r: \| r \|_2 \geq \delta: f(g(z+r)) \neq f(g(z)) \right).}_{\text{$1 - $ probability in Theorem \ref{thm:image_space_bounds} with $\omega$ identity.}}
	\end{equation}
	For example, when $\kappa$ is set to $0$, we recover the exact bounds in Theorem \ref{thm:image_space_bounds}. When $\kappa > 0$, we have to account for the use of a probabilistic definition of the modulus of continuity in the bound; this exactly corresponds to the additive $\kappa$ term in the probability in Eq. (\ref{eq:new_upper_bound}).
	
	In practice, for a fixed target probability (set to $0.25$ in the experiments of the main paper), it is possible to choose the value of $\delta$ that yields the best bound, since Eq. (\ref{eq:new_upper_bound}) is valid for any $\delta$. For a fixed value of $\delta$, we used  gradient descent (until the loss function stabilizes) in order to solve the optimization problem $\sup_{z: \| z' - z \|_2 \leq \delta} \| g(z) - g(z') \|$. For a fixed value of $\delta$, we hence summarize the procedure used to evaluate the upper bound in Algorithm \ref{alg:numerical_evaluation_upper_bound}. We have used in practice $100$ samples to estimate the upper bound, for each value of $\delta$. For any value of $\delta$, Algorithm \ref{alg:numerical_evaluation_upper_bound} provides an estimate of the upper bound; such an estimate can be improved by using many different values of $\delta$.

	\begin{algorithm}[ht]
\small
\caption{\label{alg:numerical_evaluation_upper_bound}
Numerical evaluation of the upper bound.}
\begin{algorithmic}[1]
\State // \textbf{input}: $\delta$, target probability $p_t$.
\State // \textbf{output}: numerical upper bound.
\State $p \gets p_t - p_u(\delta)$. // $p_u(\delta)$ is the probability from Theorem \ref{thm:image_space_bounds} with $\omega$ set to identity.
\Repeat: $i=1, \dots$
\State Sample $z_i \sim \mathcal{N} (0, I_d)$.
\State Compute $s_i\gets  \text{sup}_{z': \| z_i - z' \|_2 \leq \delta} \| g(z_i) - g(z') \|$.
\Until enough samples are taken
\State Use the above $s_i$ to estimate $\alpha$ such that $\tilde{\mathbb{P}} \left( s_i \geq \alpha \right) \leq p$, where $\tilde{\mathbb{P}}$ is the empirical probability distribution.
\Return $\alpha$.
\end{algorithmic}
\end{algorithm}
    
    \subsection{Illustration of generated images}

Fig. \ref{fig:svhn_image_examples} illustrates generated images for SVHN, as well as corresponding perturbed images that fool a ResNet-18 classifier (\textit{in-distribution} robustness). Similarly, Fig.   \ref{fig:cifar10_images} illustrates examples of generated images for CIFAR-10, as well as perturbed samples required to fool the VGG classifier, where perturbed images are constrained to belong to the data distribution (i.e., \textit{in-distribution} setting).

\begin{figure}[ht]
    	\centering
    	\includegraphics[scale=0.2]{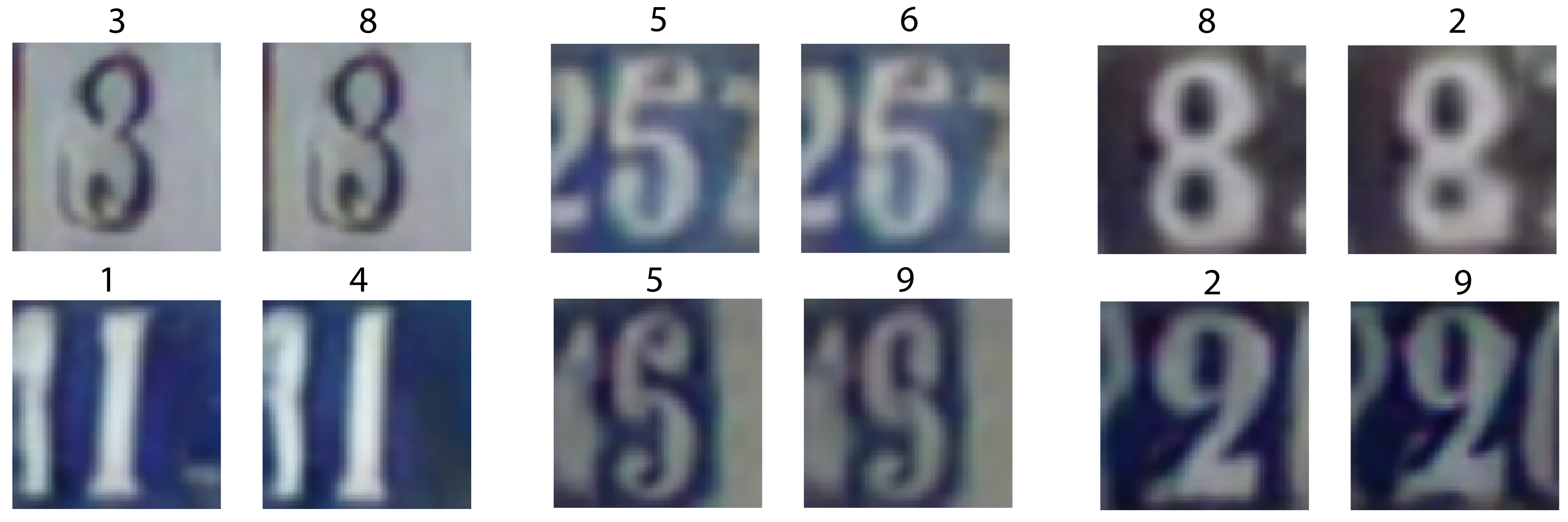}
    	\caption{\label{fig:svhn_image_examples} Examples of generated images with DCGAN for the SVHN dataset, and associated perturbed images (\textit{in-distribution} perturbations). For each pair of images, the left shows the original image, and the right shows the perturbed image. The estimated label (using ResNet-18) of each image is shown on top of each image.}
\end{figure}
    
\begin{figure}[ht]
	\centering
	\includegraphics[scale=0.2]{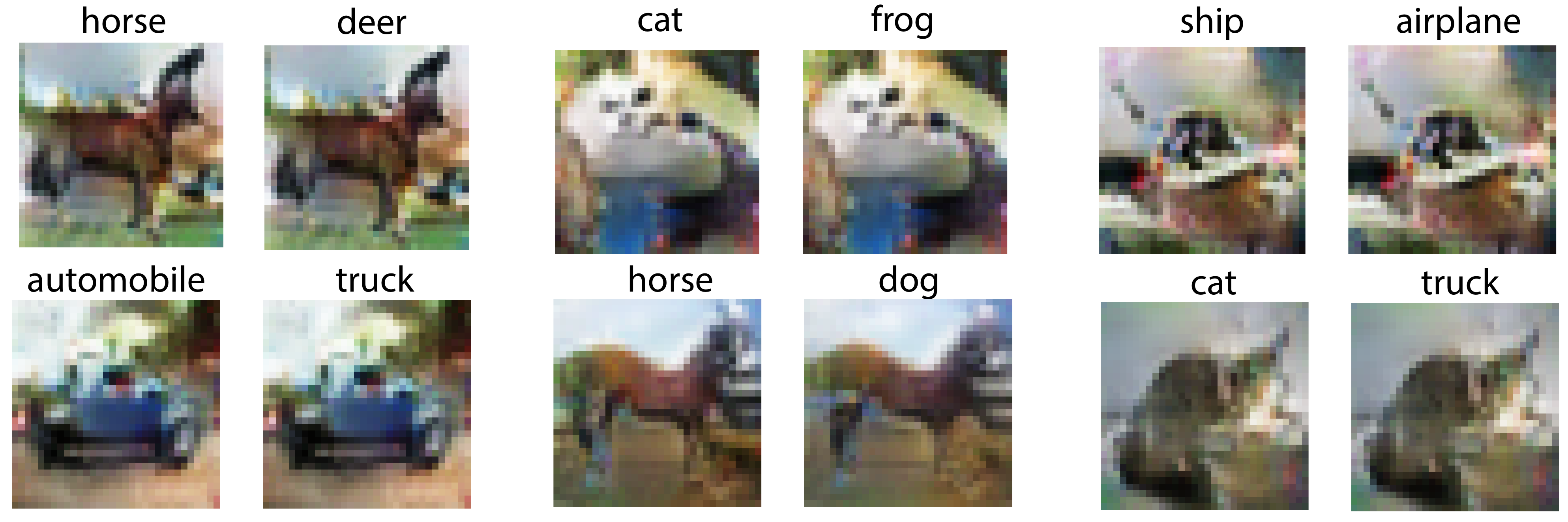}
	\caption{\label{fig:cifar10_images} Examples of generated images with DCGAN, and associated perturbed image (in-distribution perturbation). For each pair of images, the left shows the original image, and the right shows the perturbed image. The estimated label (using the VGG-type convnet) of each image (original and perturbed) is shown on top of each image.}
\end{figure}

\small{
\bibliographystyle{ieeetr}
\bibliography{bibliography}
}

\end{document}